\pdfoutput=1

\documentclass[11pt]{article}

\usepackage[final]{template}

\usepackage{times}
\usepackage{latexsym}

\usepackage[T1]{fontenc}

\usepackage[utf8]{inputenc}

\usepackage{microtype}

\usepackage{inconsolata}

\usepackage{graphicx}

\makeatletter
\def\th@definition{
  \normalfont }

\usepackage{commands/comments}
\usepackage{commands/macros}
\usepackage{commands/mmacros}

\usepackage{adjustbox}
\usepackage{float}
\usepackage{booktabs}
\usepackage{algorithm}
\usepackage{algpseudocode}
\usepackage{algorithmicx}
\usepackage{multirow}
\usepackage{tcolorbox}
\usepackage{listings}

\title{
    \benchmarkNameNemph{}\includegraphics[scale=0.065]{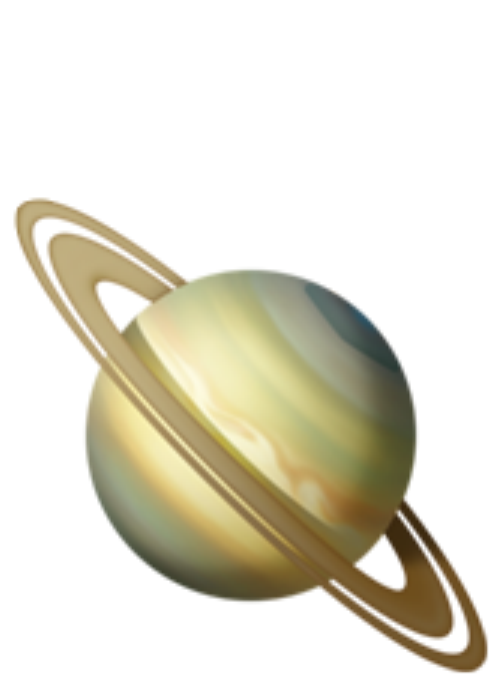}: 
    A Rigorous Benchmark for \\ Translating Text to Structured Planning Languages
}

\author{
    \textbf{Max Zuo}\thanks{These authors contributed equally to this work.} \quad
    \textbf{Francisco Piedrahita Velez} \footnotemark[1] \quad
    \textbf{Xiaochen Li} \quad \\
    \textbf{Michael L.~Littman} \quad
    \textbf{Stephen H.~Bach} \\
    $^{1}$Department of Computer Science, Brown University \\
    \texttt{ \{zuo, francisco, xiaochen\_li, michael\_littman, stephen\_bach\}@brown.edu }
}

\begin{document}

\maketitle
 \begin{abstract}
    Recent works have explored using language models for planning problems. 
    One approach examines translating natural language descriptions of planning tasks into structured planning languages, such as the planning domain definition language (PDDL). 
Existing evaluation methods struggle to ensure semantic correctness and rely on simple or unrealistic datasets.
To bridge this gap, we introduce \benchmarkName, a benchmark designed to evaluate language models' ability to generate PDDL code from natural language descriptions of planning tasks. 
    \benchmarkName{} features a novel PDDL equivalence algorithm that flexibly evaluates the correctness of generated PDDL, along with a dataset of 145,918 text-to-PDDL pairs across 73 unique state combinations with varying levels of difficulty. 
    Finally, we evaluate several API-access and open-weight language models that reveal this task's complexity.
    For example, 96.1\% of the PDDL problem descriptions generated by GPT-4o are syntactically parseable, 94.4\% are solvable, but only 24.8\% are semantically correct, highlighting the need for a more rigorous benchmark for this problem.
\end{abstract} 
\section{Introduction}

\begin{figure}[t!]
    \centering
\includegraphics[width=0.98\linewidth]{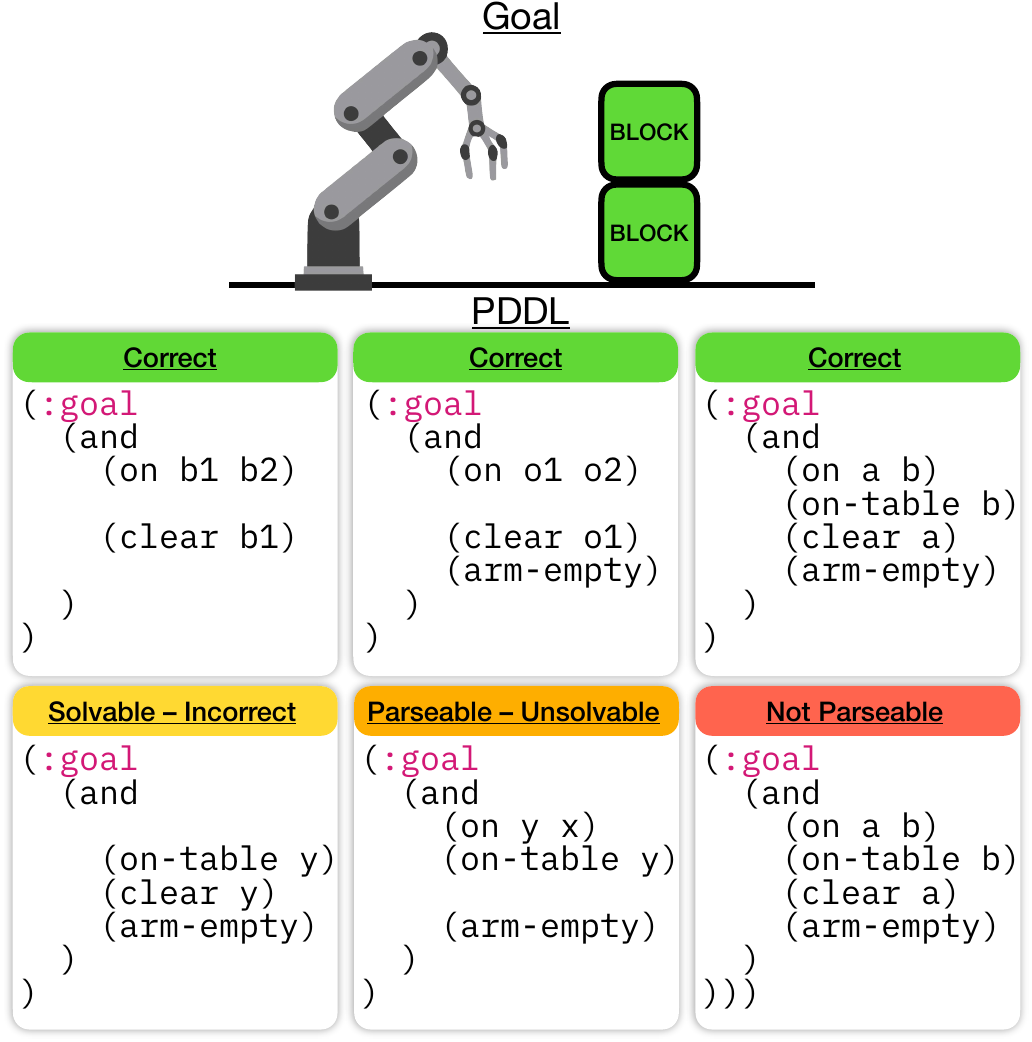} \caption{An example of one planning goal corresponding to many correct PDDL goals. All PDDL goals in the top row represent the displayed goal correctly. The bottom row illustrates PDDL goals with different error types, showing instances that are solvable (a planner can generate a plan, but for a different planning problem), parseable (the PDDL syntax is correct but will not produce any plan from a planner), and not parseable (it is not valid PDDL).
    See Section~\ref{sec:eval} for details.}
    \label{fig:states-to-pddl}
\end{figure}

Recently, there has been growing interest in using large language models (LLMs) to solve planning problems.
Some research has focused on generating plans directly with LLMs \cite{VALMEEKAM-2023-PLANBENCH, VALMEEKAM-2023-LLM-CAN-NOT-PLAN, VALMEEKAM-2023-PLANNING-ABILITIES-OF-LLM, SILVER-2022-PDDL-PLANNING-LLM, SILVER-2023-GENERALIZED-PLANNING-LLM}.
However, this approach has shown limited success; GPT-4 only achieves 35\% accuracy on simple planning problems \cite{VALMEEKAM-2023-PLANBENCH}.
Another line of research uses LLMs to convert natural language prompts into structured planning languages, such as the planning domain definition language \cite{LIU-2023-LLM+P, XIE-2023-PLANNING-GOALS, GUAN-2023-LLM-WORLD-MODELS, CHALVATZAKI-2023-LLM-PLANNING-SCENE-GRAPHS, YANG-2023-LLMS-LOGIC-PROGRAMMING}.
Early evidence suggests this method performs better than generating plans directly with LLMs \cite{LIU-2023-LLM+P}.
Despite its promise, there are no rigorous techniques or benchmarks for evaluating the translation of natural language planning descriptions to PDDL.

Translating natural language to PDDL enables a hybrid, best-of-both-worlds approach.
The LLM is responsible for interpreting natural language, and the resulting PDDL can be given to traditional, symbolic planners that have been developed over decades \cite{FIKES-1971-STRIPS, MCDERMOTT-1998-PDDL, HELMERT-2006-FAST-DOWNWARD, BAIER-2009-HEURISTIC-SEARCH-PLANNING, VIDAL-2006-BRANCHING-AND-PRUNNING}.
These traditional planners are efficient and ensure the correctness of their solutions. 
In contrast, leaving the planning to LLMs does not guarantee correctness, making them unreliable for critical applications.
Despite this shortcoming, LLMs are an enticing approach, as using traditional planners requires domain expertise and expertise in modeling planning problems.
By using LLMs to translate natural language into PDDL, we can better leverage the strengths of existing planners.

Using LLMs to translate natural language to PDDL is an instance of a code generation task.
Evaluating code generation tasks, in general, is highly challenging.
Some benchmarks for code generation use match-based metrics that look for segments that overlap with ground truth code \cite{PAPINENI-2002-BLEU, LIN-2004-ROUGE, REN-2020-CODEBLEU}, but match-based metrics cannot cover the vast space of equivalent programs \cite{CHEN-2021-LLM-CODE}.
Therefore, benchmarks often test functional correctness using a suite of unit tests \cite{ROZIERE-2020-UNSUPERVISED-CODE-GEN,KULAL-2019-SPOC}.
For planning problems, existing work uses ``plan validators'' to check if the generated code can be solved with a traditional planner \cite{LIU-2023-LLM+P, SILVER-2022-PDDL-PLANNING-LLM, SILVER-2023-GENERALIZED-PLANNING-LLM, GUAN-2023-LLM-WORLD-MODELS}.
We argue using validators alone is insufficient to determine if PDDL generation is correct.
This is because the LLM can generate valid PDDL that has \emph{nothing} to do with the user's instructions and still be considered correct---a false positive.

A rigorous evaluation of LLMs as generators of structured planning languages requires a precise definition of what it means for generated code to be correct.
This is hard because many instances of PDDL can represent the same planning problem, but it is not always obvious when they are equivalent (Figure~\ref{fig:states-to-pddl}).
Properly checking PDDL equivalence requires symbolic interpretation.

To address this challenge, we introduce \benchmarkName, a benchmark to evaluate LLMs on translating natural language descriptions of planning problems into PDDL. Our contributions are as follows:

\paragraph{Rigorous Evaluation of PDDL Equivalence.} We formally define planning problem equivalence and create an algorithm for checking whether two PDDL problems satisfy this definition.
This algorithm transforms the PDDL code into scene graphs, computes an expansion of the goal states for both PDDL problems, and then performs isomorphism checks between the graphs.
Our method ensures two PDDL problems match if and only if they represent the same underlying planning task.
We show how to make this algorithm efficient for three domains: Blocks World, Gripper, and a slightly simplified Floor Tile \cite{MCDERMOTT-2000-IPC,LOPEZ-2015-IPC-7}.

\paragraph{Benchmark Data for PDDL Generation.}
We present a dataset based on the International Planning Competition (IPC) \cite{MCDERMOTT-2000-IPC, VALLATI-2014-IPC, TAITLER-2024-IPC, SEIPP-2022-PDDL-GEN}, crafting 145,918 ground truth PDDL problems and corresponding text descriptions capturing a range of planning problems.
Each task varies in two dimensions to assess the difficulty of PDDL generation: abstraction and size.

\paragraph{Broad Evaluation of Current LLMs.} Finally, we evaluate a range of API-access and open-weight LLMs on \benchmarkName.
We evaluate in both a zero-shot setting and after fine-tuning.
We find that this task is very challenging.
GPT-4o in a zero-shot setting gets only 24.8\% correct.
Instances with abstract descriptions or many propositions are particularly challenging.
\benchmarkName\ can, therefore, serve as a benchmark of progress on this important problem.
To support future development and evaluation of LLMs, we release all the code and data for \benchmarkName\ at \href{https://github.com/BatsResearch/planetarium}{\texttt{github.com/BatsResearch/planetarium}}.
 \section{Related Work}
    
\noindent \textbf{LLMs as Planners.}
There is growing interest in using LLMs to solve planning problems.
Several papers have demonstrated the potential of using LLMs for decision-making~\cite{SHARMA-2022-SKILL-INDUCTION-PLANNING, REN-2023-UNCERTAINTY-ALIGNMENT-LLM-PLANNERS, ICHTER-2022-GROUNDING-LANGUAGE, SING-2022-PROG-PROMPT}.
Some techniques that enable the use of LLMs as planners involve decomposing abstract tasks into concrete, atomic executable actions for an agent to perform~\cite{HUANG-2022-INNER-MONOLOGUE, HUANG-2022-LLS-ZERO-SHOT-PLANNERS, SHARMA-2022-SKILL-INDUCTION-PLANNING}.
Other approaches generate plan actions by scoring the possible next steps in a sequence of actions~\cite{ICHTER-2022-GROUNDING-LANGUAGE, REN-2023-UNCERTAINTY-ALIGNMENT-LLM-PLANNERS}.
When using LLMs to generate plans directly, the LLM is given a natural language representation of a planning problem with the goal of generating a plan. 
We refer to this line of work as ``\emph{LLMs as planners}.''
One of the main findings is that LLMs have limited ability to generate and validate plans on their own, even for simple planning tasks~\cite{VALMEEKAM-2023-LLM-CAN-NOT-PLAN, SILVER-2022-PDDL-PLANNING-LLM, VALMEEKAM-2023-PLANNING-ABILITIES-OF-LLM}. 

\paragraph{Planner-Augmented LLMs.}
Given LLMs' poor performance in classical planning tasks \cite{VALMEEKAM-2023-PLANBENCH, VALMEEKAM-2023-LLM-CAN-NOT-PLAN, SILVER-2022-PDDL-PLANNING-LLM}, new approaches extend LLMs by incorporating classical planning techniques \cite{LIU-2023-LLM+P, GUAN-2023-LLM-WORLD-MODELS}. 
Some techniques frame the problem as a machine translation task, transforming a natural language text description of a planning problem into PDDL to use with classical planners \cite{LIU-2023-LLM+P,  GUAN-2023-LLM-WORLD-MODELS, XIE-2023-PLANNING-GOALS}.
We refer to this line of work as ``\emph{Planner-Augmented LLMs}.''
Other similar approaches translate natural language into alternative representations of the planning problem, such as finite state automata \cite{YANG-2023-FSA-LLM} or logic programming \cite{YANG-2023-LLMS-LOGIC-PROGRAMMING}, to solve them.

\paragraph{Benchmarking LLMs on Planning Tasks.}
Significant efforts have been made to develop benchmarks to assess the use of ``\emph{LLMs as Planners}.''
PlanBench~\cite{VALMEEKAM-2023-PLANBENCH} evaluates LLMs on various planning-related tasks, including plan generation, cost-optimal planning, plan verification, and others. 
Their work focuses on evaluating LLMs on their ability to generate plans.
In contrast, \benchmarkName\ focuses on evaluating  ``\emph{Planner-Augmented LLMs}''.
Here, the goal is to determine whether an LLM can successfully translate natural language descriptions of planning problems into the correct PDDL representations of those problems.
This is arguably a more natural use of LLMs, but it is harder to evaluate because the output is an internal representation and not the end result of the planning process.

\paragraph{Planning problem equivalence.} There is also work exploring techniques for determining the equivalence of planning problems, such as \citet{SHRINAH-2021-D-VAL}.
\citet{CHRPA-2023-PLANNING-DOMAIN} and \citet{SIEVERS-2021-LIFTED-PDDL-TASKS} use a lifted graph representation of the planning domain, with the former specifically using graph isomorphism to check equivalence. Planetarium also employs graph isomorphism, but our focus is on evaluating the equivalence of PDDL problem instances rather than domains. Since domains and problem instances differ, the representation and isomorphism checks must also be adapted accordingly.

 \section{Preliminaries}

To present \benchmarkName, we first introduce planning, PDDL, and scene graphs.

\subsection{Classical Planning Problems}

We use the set-theoretic form of classical planning problems~\cite{GHALLAB-2004-UTOMATED-PLANNING}.

\begin{definition}
    \label{def:classical-planning}
    A planning problem $\planningProblem$ is denoted by the tuple $(\propositionSymbolSpace, \stateSpace, \actionSpace, \transitionFunction, \initialState, \goalPropositions)$, where:
    
    \begin{itemize}
        \item $\propositionSymbolSpace$ is a finite set of proposition symbols representing different facts about the world.
        \item $\stateSpace \subseteq 2^{\propositionSymbolSpace}$ is a set of states. Each state $\state \subseteq \propositionSymbolSpace$ is the set of true propositions in that state.
        \item $\actionSpace$ is the set of actions. Each action $a$ is a triplet $(\actionPreconditions[a],~\actionPositiveEffects[a],~\actionNegativeEffects[a]$, where $\actionPreconditions[a]$ is the set of preconditions that must be satisfied for $a$ to be executed, $\actionPositiveEffects[a]$ is the set of positive effects that become true after executing $a$, and $\actionNegativeEffects[a]$ is the set of negative effects that become false after executing $a$.
        \item The transition function $\gamma(s, a) = (s - \texttt{del}(a)) \cup \texttt{add}(a)$ models how the world changes by an action.
        \item $\initialState$ is the initial state of the world from which the problem begins.
        \item $\goalPropositions \subseteq \propositionSymbolSpace$ represents the goal propositions, indicating which propositions must be true for a state to be considered a goal state. The set of goal states is $\goalStates = \{ \state \in \stateSpace \; | \; \goalPropositions \subseteq \state \}$.
    \end{itemize}
\end{definition}

A plan $\plan = < \action[1], \action[n], \ldots, \action[n] >$ is a sequence of actions that leads from $\initialState$ to any state in $\goalStates$ following $\transitionFunction$. This plan is a \emph{solution} to $\planningProblem$.

\subsection{Planning Domain Definition Language}

PDDL is a specialized language designed to provide a unified way to represent planning problems. It can represent various types of planning problems, including classical planning problems~\cite{MCDERMOTT-2000-IPC, GEREVINI-2006-PDDL3}. A PDDL planning problem consists of two files: the domain file, which describes the constant parts of the world model, and a problem file, which specifies particular planning scenarios using the world model outlined in the domain file.

The domain file specifies the set of possible actions $\actionSpace$ and their preconditions and effects, which collectively define the transition function $\transitionFunction$.
The problem file defines the initial state $\initialState$ and the goal propositions $\goalPropositions$.
Finally, the set of proposition symbols $\propositionSymbolSpace$ is created by combining the predicates from the domain file with the objects defined in the problem file.

\benchmarkName\ focuses on classical planning problems within the STRIPS subset of PDDL~\cite{MCDERMOTT-1998-PDDL, FIKES-1971-STRIPS}. 
STRIPS provides the basic grammar for describing actions by specifying a set of preconditions that must be met for an action to be applicable and a set of effects that modify the propositions that are true after the action's execution.  
The expressiveness of problems defined in STRIPS is equivalent to those characterized by the set-theoretic definition (Definition~\ref{def:classical-planning}) of classical planning problems~\cite{GHALLAB-2004-UTOMATED-PLANNING}.

\subsection{Scene Graphs}
\label{sec:scene-graph}

To compare different planning problems, we use scene graphs. 
A scene graph is a data structure commonly used in fields such as computer vision and graphics~\cite{JOHNSON-2015-IMAGE-RETRIEVAL-SCENE-GRAPHS, CHANG-2023-SCENE-GRAPHS-SURVEY}, rearrangement~\cite{RAMACHANDRUNI-2023-CONSOR}, and planning to represent objects, their attributes, and the relationships among them.
In our work, we define scene graphs as directed graphs with types and attributes for both nodes and edges.
We represent a PDDL problem file with scene graphs as follows.
We create one scene graph (the \emph{initial scene}) for the initial state and another for the set of goal propositions (the \emph{goal scene}).
For every object, we create a node with an \texttt{object} type.
Then, for every proposition that is listed in the problem file, we create a node with a \texttt{proposition} type.
That node is given an attribute with the name of its predicate.
Then, for each argument to the predicate in that proposition, we add an edge from the proposition node to the corresponding object.
Each edge is given three attributes: the name of the predicate, the position of the argument (first, second, etc.), and whether it is defined in the initial state or the goal propositions.
A scene graph is thus $(\{O \cup P\}, E)$, where $O$ is the set of nodes with \texttt{object} type, $P$ is the set of nodes with \texttt{Proposition} type, and $E$ is the set of edges.

\newcommand{\SceneGraph}{\textit{\textrm{SceneGraph}} }
\newcommand{\init}{\textit{\textrm{init}} }
\newcommand{\goal}{\textit{\textrm{goal}} }
\newcommand{\ProblemGraph}{\textit{\textrm{ProblemGraph}} }

We further define a \emph{problem graph} as the combination of an initial scene and goal scene.
Given $\SceneGraph_{\init} = (\{O \cup P_{\init} \}, E_{\init})$ and $\SceneGraph_{\goal} = (\{O \cup P_{\goal}\}, E_{\goal})$, then a problem graph merges the scene graphs such that $\ProblemGraph = (\{O \cup P_{\init} \cup P_{\goal}\}, \{E_{\init} \cup E_{\goal}\})$.
We define graph isomorphism on any of these graphs as an edge and type-preserving bijection between nodes, meaning that two graphs share a connectivity structure where all types and attributes match.
See Appendix~\ref{app:examples} for examples of diagrams of scene and problem graphs. \section{Evaluation Framework}

In this section, we describe the design of our benchmark.
\benchmarkName\ consists of two components: an algorithm that validates whether two PDDL problem files, a ground truth file and an LLM-generated file, are equivalent, and a curated dataset of planning tasks against which an LLM can be evaluated. \subsection{Planning Problem Equivalence} \label{sec:planning}
The first step to benchmarking PDDL generation is determining how to decide whether the generated code matches ground truth code.
One might assume that checking if two PDDL problem files are equivalent is straightforward.
However, the same goal state could be represented by many PDDL problem files, as shown in Figure~\ref{fig:states-to-pddl}.

Given these difficulties, we propose a definition of equivalence in terms of classical planning problems.
The main idea is to find a \emph{bijective function} between the sets of proposition symbols $\propositionSymbolSpace$ of the two problems such that makes the two problems equal.
Our definition assumes that the transition function $\transitionFunction$ is shared between the two problems.
Our formal definition of equivalence between two planning problems follows.

\begin{definition}
    \label{def:equivalence}
    Two planning problems $\planningProblem[1] = (\propositionSymbolSpace[1], \stateSpace[1], \actionSpace[1], \transitionFunction, \initialState[1], \goalPropositions[1])$ and $\planningProblem[2] = (\propositionSymbolSpace[2], \stateSpace[2], \actionSpace[2], \transitionFunction, \initialState[2], \goalPropositions[2])$ with the same transition function $\transitionFunction$ are \textbf{equivalent} if there exists a bijective function $\isomorphicMapping: \propositionSymbolSpace[1] \to \propositionSymbolSpace[2]$ such that:
    \begin{enumerate}
        \item $\stateSpace[2] = \{ \{\isomorphicMapping(\propositionSymbol) : \propositionSymbol \in \state  \} : \state \in \stateSpace[1] \}$ \label{itm:property-1}
        \item $\initialState[2] = \{\isomorphicMapping(p) : p \in \initialState[1] \}$ \label{itm:property-2}
        \item $\goalStates[2] = \{ \{\isomorphicMapping(\propositionSymbol) : \propositionSymbol \in \state  \} : \state \in \goalStates[1] \}$ \label{itm:property-3}
        \item \hspace{-0.5em}\begin{minipage}[t]{\linewidth}
            \vspace{-0.7\baselineskip}
            $\begin{array}{lcl}
                \actionSpace^2 = \{ ( &\{\isomorphicMapping(p): p \in \texttt{pre}(\action)\}&, \\ 
                                      &\{\isomorphicMapping(p): p \in \texttt{add}(\action)\}&, \\
                                      &\{\isomorphicMapping(p): p \in \texttt{del}(\action)\}&) : \action \in \actionSpace^1 \} .
            \end{array}$  
            \end{minipage} \label{itm:property-4} 
    \end{enumerate}
\end{definition}

\begin{figure}[t!]
    \centering
    \includegraphics[width=0.96\linewidth]{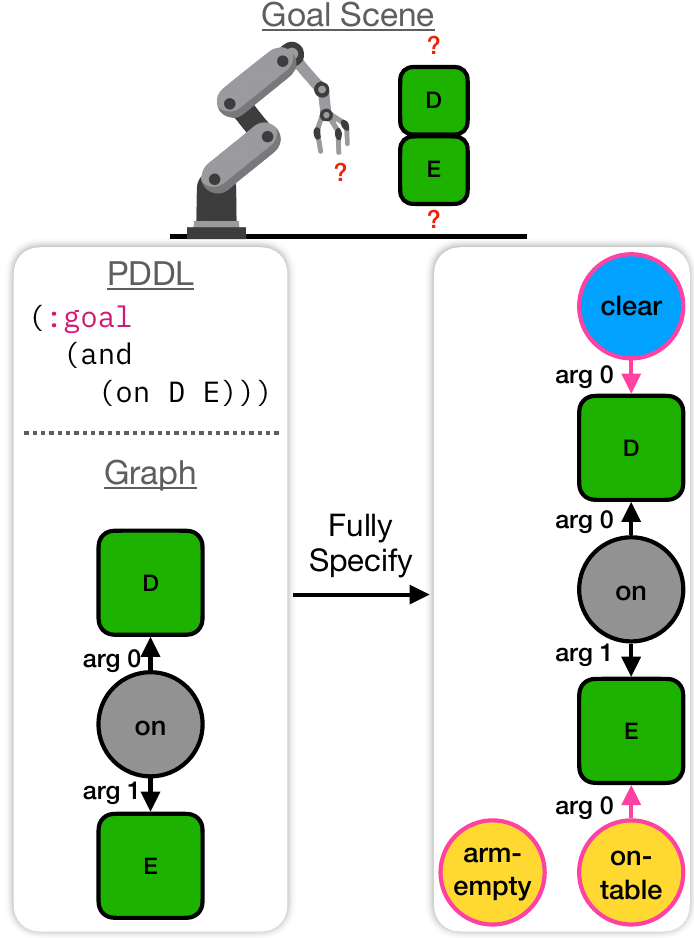}
    \caption{A demonstration of how \texttt{fullySpecify} fills in gaps of the goal state of a planning problem. 
}
    \label{fig:fully-specify}
\end{figure}

That is, all predicates across the states, action definitions, initial state, goal specification, and actions of one problem are replaced with their corresponding predicates from the other.

This definition is not directly usable for checking equivalence between two PDDL problem files because it relies on finding a bijection between the sets $\propositionSymbolSpace$ of each problem, and PDDL does not define these sets directly.
To build the set $\propositionSymbolSpace$ with PDDL, one needs to use the predicates in the domain file and the objects in the problem file, instantiating each predicate with all possible objects they might take.
We assume the entire PDDL domain is shared between the problems.
This fact entails that the predicates will be shared, making it necessary to look for bijective functions for PDDL only over the objects.
The challenge is that each PDDL problem file can correspond to many pairs of initial and goal states because the set of goal propositions makes an open-world assumption.
That set can leave implicit trivial propositions that are necessarily true and therefore do not change the underlying planning problem.
We must therefore \emph{fully specify} the goal, meaning that we identify all propositions that are true in all reachable goal states when starting from the initial state.

Our algorithm for checking equivalence is summarized in Algorithm~\ref{alg:planning-problem-equivalency}.
First, we transform each set of initial state and goal state propositions into scene graphs.
Second, we fully specify the goal scene graphs by adding all trivially true edges.
Finally, we join the initial state scene graph with each goal state graph to create problem graphs and look for a bijection between objects such that the problem graphs are isomorphic.

\begin{algorithm}
    \caption{Planning Problem Equivalence} 
    \label{alg:planning-problem-equivalency}
    \begin{algorithmic}[1]
        \Function{equivalent}{$\text{P}_a, \text{P}_b, \text{isPlaceholder}$}
            \State $\text{s}_{i,a}, \text{g}_a \gets \text{toSceneGraph}(\text{P}_a)$
            \State $\text{s}_{i,a}, \text{g}_a \gets \text{toSceneGraph}(\text{P}_b)$

            \If{$\text{canDoFast}(\text{s}_{i,a}, \text{g}_a, \text{s}_{i,b}, \text{g}_b)$}
                \State \Return $\text{fastEquivalent}(\text{s}_{i,b}, \text{g}_a, \text{s}_{i,a}, \text{g}_b)$
            \EndIf

            \State $\text{g}^\star_a \gets \text{fullySpecify}(\text{s}_{i,a}, \text{g}_a)$
            \State $\text{g}^\star_b \gets \text{fullySpecify}(\text{s}_{i,b}, \text{g}_b)$

            \If{$\text{isPlaceholder}$}
                \State $\text{init\_iso} \gets \text{isIsomorphic}(\text{s}_{i,a}, \text{s}_{i,b})$
                \State $\text{goal\_iso} \gets \text{isIsomorphic}(\text{g}^\star_a, \text{g}^\star_b)$
                \State \Return $\text{init\_iso} \wedge \text{goal\_iso}$
            \Else
                \State $\text{p}_a \gets \text{join(}\text{s}_{i,a}, \text{g}^\star_a\text{)}$
                \State $\text{p}_b \gets \text{join(}\text{s}_{i,b}, \text{g}^\star_b\text{)}$

                \State \Return $\text{isIsomorphic}(\text{p}_a, \text{p}_b)$
            \EndIf
        \EndFunction
    \end{algorithmic} 
\end{algorithm}

\paragraph{Transform to Scene Graphs.}
From each PDDL problem file, we generate two scene graphs: one for the initial state and another for the goal state (lines 2--3).
For each transformation, we first create an \texttt{object} node for each object.
Then, for each proposition in the collection, we create a new \texttt{proposition} node and create edges between the proposition node and the objects it takes as arguments.
Edge attributes denote argument order, predicate type, and whether the proposition is in an initial or goal scene.
See Section~\ref{sec:scene-graph} for details.

\paragraph{Check Easy Cases.}
For speed, we check several cases and return early if we can (lines 4--6).
First, if the number of \texttt{object} nodes in each graph is \emph{not} equal, the problems cannot be equivalent.
Second, if the initial scenes are \emph{not} isomorphic, then the problems cannot be equivalent.
Finally, if the problem graphs, composed of the initial and goal scenes, are isomorphic, then the problems are equivalent.

\begin{table*}[]
    \caption{A breakdown of the \benchmarkName\ dataset by the level of abstractness of the text description, number of propositions in the ground truth PDDL problem (size), and train/test split. Each instance consists of a ground truth PDDL problem description paired with a text description.}
    \label{tab:dataset-table}
    \centering
    \resizebox{1.0\textwidth}{!} {
        \begin{tabular}{lrrrrrrrrrrrr}
            \toprule
            \multirow{2}{*}{Domain} & \multirow{2}{*}{Total} & \multicolumn{4}{c}{Abstractness} & \multicolumn{5}{c}{Num. of Ground Truth Propositions (Size)} & \multicolumn{2}{c}{Split}\\
            \cmidrule(lr){3-6} \cmidrule(lr){7-11} \cmidrule(lr){12-13}
            &    & \multicolumn{2}{c}{Explicit to} & \multicolumn{2}{c}{Abstract to} & \multicolumn{5}{c}{} & \multicolumn{2}{c}{} \\
            &    & Explicit & Abstract & Explicit & Abstract & 1-20 & 21-40 & 41-60 & 61-80 & >80 & Train & Test\\ \midrule
            Blocks World & 85,605 & 22,595 & 20,152 & 20,152 & 22,706 & 963 & 10,011 & 47,153 & 25,766 & 1,712 & 80,612 & 4,993\\
            Gripper     & 54,141 & 13,683 & 10,565 & 12,466 & 17,427 & 2,334 & 13,640 & 16,405 & 20,325 & 1,437 & 49,363 & 4,778\\ 
            Floor Tile      &6,172 & 2,289 & 797 & 793 & 2,293 & 706 & 2,114 & 2,224 & 404 & 724 & 0 & 6,172\\
            \bottomrule
        \end{tabular}
    }
\end{table*}

\paragraph{Fully Specify the Goal Scenes.}
If the input is not an easy case, then we have to reason about the sets of goal states defined by the goal scenes (Figure~\ref{fig:fully-specify}).
Condition 3 of Definition~\ref{def:equivalence} requires that the sets of goal states that are consistent with the given goal propositions be equal after substituting matching propositions.
Since PDDL uses the open-world assumption for goals, we have to identify all propositions that are true in all reachable goal states when starting from the initial state for each problem.
The function \texttt{fullySpecify} finds all such propositions and adds them to the goal scenes as additional \texttt{proposition} nodes with the corresponding edges and attributes (lines 7--8).
We show how to implement \texttt{fullySpecify} efficiently for the three domains in \benchmarkName\ in Appendix~\ref{app:implementation}.

\paragraph{PDDL without Object Identity.}
Our algorithm runs in two modes, depending on the type of problem it is checking.
Sometimes we want to compare a generated PDDL problem file with many ground truth problem files and see if the generated file matches any one of them.
For example, if the natural language description says ``make a tower of height 3,'' the specific blocks to use are unspecified, and any permutation of blocks that builds a tower of height 3 should be considered correct.
Concretely, we want to treat the objects in the PDDL goal states as \emph{placeholders} and accept any permutation of them.
We check this condition when \texttt{isPlaceholder} is \textbf{True} (determined by the problem type) simply by checking isomorphism between initial and goal scenes separately rather than combining them into problem graphs (lines 10--12).

\paragraph{PDDL with Object Identity.}
If isPlaceholder is \textbf{False}, we want to check whether the problem files are precisely equivalent under Definition~\ref{def:equivalence}, meaning that the objects in the initial scenes correspond to the same objects in the goal scenes as well.
We first join the corresponding initial and scene graphs into problem graphs as described in Section~\ref{sec:scene-graph} (lines 14--15).
Then we check whether the two problem graphs are isomorphic (line 16).

We illustrate Algorithm~\ref{alg:planning-problem-equivalency} with examples in Appendix~\ref{app:examples}.
We also formally state its correctness, with the proof in Appendix~\ref{app:proof}.
\begin{theorem}
\label{thm:main}
\texttt{Equivalent}($P_a$, $P_b$, \textbf{False}) returns \textbf{True} if and only if the PDDL problem files $P_a$ and $P_b$ represent equivalent planning problems under Definition~\ref{def:equivalence}.
\texttt{Equivalent}($P_a$, $P_b$, \textbf{True}) returns \textbf{True} if and only if $P_a$ represents a planning problem that is equivalent to some planning problem represented by $P_b$ after a permutation of the objects in its goal state.
\end{theorem}
 \subsection{The Dataset} \label{sub:dataset}
The \benchmarkName\ dataset includes 145,918 text-to-PDDL pairs derived from the Blocks World, Gripper, and Floor Tile domains ~\cite{SEIPP-2022-PDDL-GEN, MCDERMOTT-2000-IPC}.
Our dataset captures 25 unique initial and goal state configurations describing 73 different tasks.
In the Blocks World domain, a robotic hand manipulates blocks arranged in various configurations on a table, with 5 predicates and 4 actions.
The Gripper domain features a robot with grippers transporting balls between rooms, with 7 predicates and 3 actions.
While versions of Blocks World and Gripper domains exist with \texttt{:typing}, we choose to use the IPC standard versions, which do not.
We simplified the standard Floor Tile domain~\cite{SEIPP-2022-PDDL-GEN}, where robots are tasked with painting tiles, by removing the \texttt{left} and \texttt{down} predicates and relying on the \texttt{right} and \texttt{up} predicates to check relative positions.
We further relax the domain by removing the \texttt{clear} predicate (which introduces path constraints) and the \texttt{free-color} predicate (which IPC problems did not use and is not modified or required by any actions in the original domain).
Finally, we added \texttt{paint-left} and \texttt{paint-right} actions, which reduced edge cases for the domain, resulting in a domain with 6 predicates and 9 actions (compared to 10 predicates and 7 actions in the IPC version).
More details about the dataset can be found in Appendix~\ref{app:dataset-details}.

\begin{table*}[]    
    \caption{Examples of  abstract and concrete descriptions for each domain in the benchmark.}
    \centering
    \resizebox{1.0\textwidth}{!} {
        \begin{tabular}{@{}lll@{}}
        \toprule
        Domain &
          Abstract &
          Explicit \\ \midrule
        blocksworld &
          \begin{tabular}[c]{@{}l@{}}You have 3 blocks, stacked into 2 towers of heights\\  1, 2, and your arm is empty.  Your goal is to invert \\ each individual stack of blocks, such that the block \\ that in each tower that was originally on the bottom \\ will be on the top.\end{tabular} &
          \begin{tabular}[c]{@{}l@{}}Your arm is empty. b1 is clear. b1 is on the table.\\ b2 is clear. b2 is on b3. b3 is on the table. Your goal \\ is to have the following: Your arm should be empty. \\ b3 should be clear. b3 should be on b2. b2 should be \\ on the table. b1 should be clear. b1 should be on the \\ table.\end{tabular} \\
        gripper &
          \begin{tabular}[c]{@{}l@{}}You have 2 rooms, 2 balls, and 2 grippers. 1 balls are \\ distributed across the same number of grippers, and \\ the rest are in the first room. The robby is in the first \\ room. Your goal is to gather all balls into one room.\end{tabular} &
          \begin{tabular}[c]{@{}l@{}}You have 2 rooms, 2 balls, and 2 grippers. 1 balls \\ are distributed across the same number of grippers, \\ and the rest are in the first room. The robby is in the \\ first room. Your goal is to have the following:  gripper1 \\ should be free. gripper2 should be free. ball1 should be \\ at room1. ball2 should be at room1.\end{tabular} \\
        floor-tile &
          \begin{tabular}[c]{@{}l@{}}You have 1 robots, 2 colors, and 2 unpainted tiles \\ arranged in a grid with 1 rows and 2 columns. The \\ first robot is at the top-left corner, and has the first \\ color. All colors are available. Your goal is to paint \\ all the tiles with the same color.\end{tabular} &
          \begin{tabular}[c]{@{}l@{}}You have 1 robot. You have 2 tiles. You have 2 colors.\\ Tile tile2 is to the right of tile tile1. The robot robot1\\ has color color1. The robot robot1 is at tile tile1. Color color1\\ is available. Color color2 is available. Your goal is to have\\ the following: Tile tile1 should be painted with color color1.\\ Tile tile2 should be painted with color color1.\end{tabular} \\ \bottomrule
        \end{tabular}
    }
\end{table*} 
\paragraph{Dataset Construction.}
We focus on these three domains commonly used in other works~\cite{VALMEEKAM-2023-PLANBENCH, LIU-2023-LLM+P} but are nevertheless challenging. 
Each entry in our dataset is a text-to-PDDL pair consisting of a text description detailing the initial and goal states and the corresponding ground truth PDDL.
(See Appendix~\ref{app:text-to-PDDL-pair} for examples.)
The composition of the \benchmarkName{} dataset is presented in Table~\ref{tab:dataset-table}, which categorizes the data according to text description abstractness and the number of propositions.
The table further delineates the train/test split, illustrating the distribution of data reserved for evaluation purposes. For a comprehensive overview of the task descriptions and the specific allocation of tasks between training and test sets, readers are directed to Appendix~\ref{app:dataset-details}. The dataset is procedurally generated: we handcraft templates for each task configuration, which we combine with one another to generate problems at scale.

We vary the data along two dimensions: abstractness (explicit vs.\ abstract) and size.
Explicit planning problem text descriptions correspond directly to propositions found in the problem PDDL (e.g., ``block 1 is on block 2''). Abstract text descriptions instead summarize a state (e.g., ``all blocks are in a single tower'').

Since our text descriptions contain both initial and goal states, each can be either an abstract or explicit description.
This leads to four possible abstractness categories: explicit to explicit, explicit to abstract, abstract to explicit, and abstract to abstract.

We measure the size of a problem by the number of propositions listed in the ground truth problem PDDL.
Larger problems typically pose greater challenges for LLMs. \begin{figure*}[h!]
    \begin{centering}
        \includegraphics[width=\linewidth]{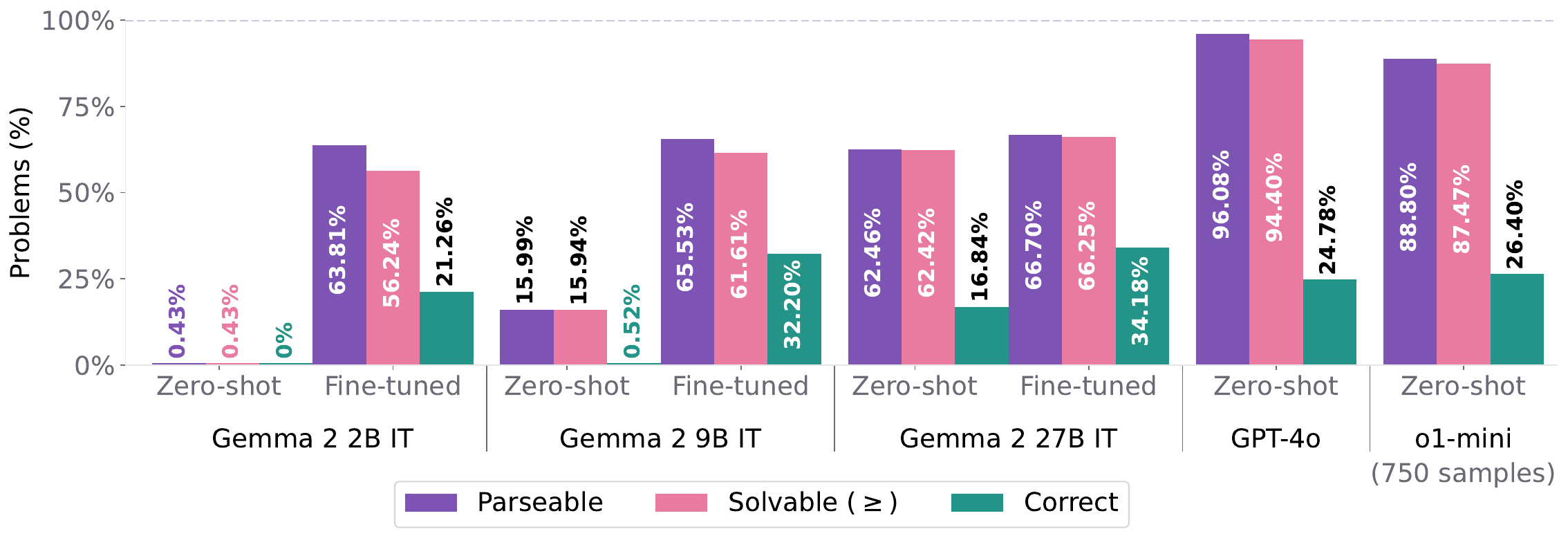}
        \caption{Performance of various models on the \benchmarkName{} test set.}
        \label{fig:results}
    \end{centering}
\end{figure*}

\section{Evaluating LLMs on \benchmarkName} \label{sec:eval}

As an initial snapshot of the field's current state, we evaluate several API-access and open-weight language models on \benchmarkName\ in both zero-shot and fine-tuned settings.
We find that while powerful models like GPT-4o can often generate valid PDDL problems (in the sense that some plan solves them), they are rarely correct.
This result underscores the need for \benchmarkName's rigorous approach to evaluation of PDDL generation.
The code to recreate the entire evaluation is available at
\href{https://github.com/BatsResearch/planetarium}{\texttt{github.com/BatsResearch/planetarium}}.

\paragraph{Models.}
We evaluate two API-access models, GPT-4o and o1-mini, and we evaluate three open-weight models before and after fine-tuning: Gemma 2 2B IT, Gemma 2 9B IT, and Gemma 2 27B IT~\cite{team2024gemma}. Details regarding fine-tuning can be found in Appendix~\ref{app:experiments}.

\paragraph{Evaluation Protocol.}
We evaluate models on the \benchmarkName{} test set, which consists of heldout configurations from the Blocks World and Gripper domains and the entirety of the Floor Tile domain.
Models are prompted with the natural language description of the task along with the respective domain PDDL.
We record three metrics for the generated problems: the number of parseable problems, the number of solvable problems, and the number of correct problems.
We say a model output is \emph{parseable} if a PDDL parser supporting \texttt{:strips} can extract a valid PDDL problem from a substring in the output and if it can be converted into our graph representation.
A problem is \emph{solvable} if it is parseable and a plan can be applied to the initial scene that results in the goal scene.
Due to the size and complexity of some of the problems in our dataset, a generalized classical planner cannot always reliably and quickly return solutions.
Instead, we built specialized planners for Blocks World and Gripper problems that work on all problems in our dataset and generally all validly defined \texttt{blocksworld} and \texttt{gripper} domain problems except for a few invalid edge cases (e.g., one block on top of two blocks at a time, holding two blocks, etc.). We use Fast Downward~\cite{HELMERT-2006-FAST-DOWNWARD} for \texttt{floor-tile} problems.
We then validate all plans with VAL~\cite{HOWEY-2004-VAL}.
We then feed parseable and solvable problems into our PDDL equivalence algorithm to verify equivalence to our ground truth PDDL and thereby determine if it is \emph{correct}.
We found our equivalence algorithm to take, on average, 12ms per example to compute on an M2 Apple Silicon laptop with batch parallelization.

\begin{figure}[t!]
    \begin{centering}
        \includegraphics[width=0.96\linewidth]{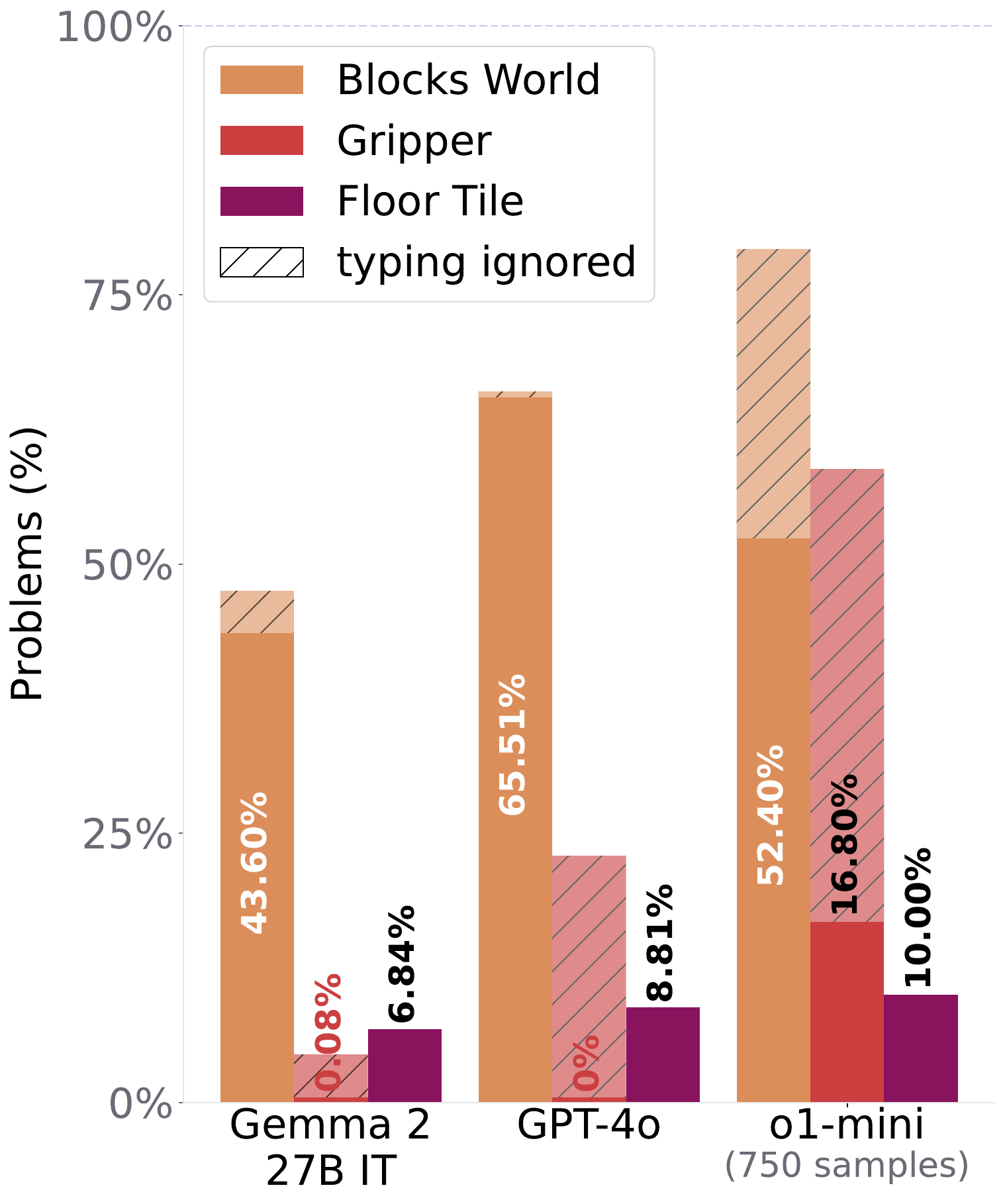}
        \caption{Breakdown of zero-shot performance by domain for Gemma 2 27B IT, GPT-4o, and o1-mini.}
        \label{fig:domain-breakdown}
    \end{centering}
\end{figure}

\paragraph{Results.} The performance of Gemma 2 2/9/27B IT, GPT-4o, and o1-mini on the \benchmarkName{} test set are shown in Figure~\ref{fig:results}.
We report the percentage of generated plans that were parseable, solvable, and correct from zero-shot and fine-tuned settings, averaged across all domains.
For o1-mini, we randomly sample 250 problems per domain to evaluate due to prohibitive costs.
In the zero-shot setting, o1-mini and GPT-4o performed the best, followed closely by Gemma 27B IT.
For all models, the percentage of generated PDDL that is semantically correct significantly lags behind the percentage of generated PDDL that is parseable and solvable.

Fine-tuning improved performance across all open-weight models, with fine-tuned Gemma 2 27B IT achieving the highest accuracy of all tested models at 34.2\%.

\paragraph{Takeaways}
We find two key takeaways from the type of mistakes these language models tended to make.
First, a significant portion of errors come from models ``ignoring'' the domain in context and referring to an incorrect domain, likely from its prior training corpus.
We can see this by the large increase in accuracy when ignoring \texttt{:typing}, shown in Figure~\ref{fig:domain-breakdown}: our Blocks World and Gripper domains do not have typing, yet when generating PDDL, these models confidently and consistently add object types.

Second, there seems to be a strong correlation between a model's performance on a certain domain and the domain's frequency in the wild: Gripper and Blocks World are often used as introductory PDDL examples, while Floor Tile is less common.
However, Gripper, and to a lesser extent, Blocks World, is often defined with \texttt{:typing}. This may explain the difference in performance between Blocks World and Gripper across models, as well as the difference in performance between enforcing versus not enforcing \texttt{:typing}.

Succinctly, we see a trend when relaxing the \texttt{:typing} requirement that models find problems in less common domains harder, accounting for everything else.
 \section{Conclusion} \label{sec:conclusion}

\benchmarkName\ is a new benchmark for assessing the ability of LLMs to translate natural language descriptions of planning problems into PDDL.
Our work reveals that language models find it challenging to generate semantically correct structured planning language descriptions. Models like GPT-4o can often produce valid, seemingly correct descriptions (94.4\%) when, in reality, only a small fraction (24.8\%) are genuinely correct.
We hope that \benchmarkName\ will drive progress on hybrid approaches combining LLMs and classic planners, setting a standard for evaluating such tasks. 

\section*{Acknowledgements}
This research is supported in part by the Office of Naval Research (ONR) award N00014-20-1-2115.
We gratefully acknowledge support from Cisco, Cognex, and the Brown Computer Science Faculty Innovators Fund.
Disclosure: Stephen Bach is an advisor to Snorkel AI, a company that provides software and services for data-centric
artificial intelligence.
 \section*{Limitations}

\benchmarkName\ has a few important limitations: \benchmarkName\ currently only supports the Blocks World, Gripper, and Floor Tile domains. While these domains have been popular for studying LLMs and their relation to planning, incorporating more expressive domains in the future will widen the scope of \benchmarkName.
This benchmark is also currently restricted to the STRIPS subset of PDDL. Extending it to support more expressive subsets of PDDL will allow us to evaluate more complex, real-world planning problems such as non-deterministic, temporal, and numeric domains. \section*{Potential Impacts}

A potential societal impact of this research is ensuring the correctness of translating natural language into structured planning languages. 
If this translation method becomes widespread but its evaluation remains inaccurate, systems could produce misleading or misaligned results that could cause harm if acted upon.
\benchmarkName\ highlights the importance of assessing the correctness of translations. We achieve this by reasoning about PDDL semantics and the inherent structure of classical planning problems.

\bibliography{custom}

\clearpage
\appendix

\appendix

\section{Proof of Theorem 1}
\label{app:proof}
\begin{proposition}
\label{prop:is-equivalent-no-placeholder}
    \texttt{Equivalent}($P_a$, $P_b$, \textbf{False}) returns \textbf{True} if and only if the PDDL problem files $P_a$ and $P_b$ represent equivalent planning problems under Definition~\ref{def:equivalence}.
\end{proposition}

\begin{proof}
    To prove this proposition, we need to prove each direction of the biconditional.
    First, we will prove the forward implication.
    There are two possible ways the algorithm can return \texttt{True}: Either in the \texttt{isIsomorphic} check (line 16) or in \texttt{fastEquivalent} (line 5).

    \noindent \textbf{Case 1:}
        When \texttt{isIsomorphic}($\text{p}_a$, $\text{p}_b$) returns \texttt{True}, we know that there exists a bijection $\varphi$ from vertices $V_{p_{a}}$ to vertices $V_{p_{b}}$. Since $\varphi$ is a type-preserving bijection, it means the two graphs share a connectivity structure where all types and attributes match.
    
        By using Lemma~\ref{lemm:bijection-transformation}, we can build a new bijection $f$ from $\varphi$. 
        Instead of mapping vertices in the scene graph, $f$ maps propositions from $\propositionSymbolSpace[a]$ in planning problem $\planningProblem[a]$ to propositions in $\propositionSymbolSpace[b]$ in planning problem $\planningProblem[b]$. 
        Given that $\varphi$ is a type-preserving bijection, we know that the bijection $f$ derived from $\varphi$ also maps objects to objects, predicates to predicates, vertices in the initial scene $s_{i, a}$ to vertices in $s_{i, b}$, and vertices in the goal scene $g_{a}^{*}$ to vertices in $g_{b}^{*}$.
    
        To prove Property~\ref{itm:property-1}, we can build the set $\propositionSymbolSpace[a]$ for each planning problem by taking each predicate in the PDDL domain and instantiating it with all possible combinations of objects from the problem PDDL $\planningProblem[a]$.
        Then, we can construct the set of all possible states of the problem by taking the power set $\stateSpace[a] = 2^{\propositionSymbolSpace[a]}$. 
        For each proposition $\propositionSymbol \in \state$ in each state $\state \in \stateSpace[a]$, we apply the bijection $f$ to show that we get $\stateSpace[b]$:

        \begin{flalign*}
            &\stateSpace = \{ \{\isomorphicMapping(\propositionSymbol) : \propositionSymbol \in \state  \} : \state \in \stateSpace[a] \} \\
            &= \{ \{\langle \varphi(q^a), \varphi(o_1^a), \ldots, \varphi(o_k^a) \rangle : \propositionSymbol \in \state  \} : \state \in \stateSpace[a] \} \\
            &= \{ \{\langle q^b, o_1^b, \ldots, o_k^b \rangle : \propositionSymbol \in \state  \} : \state \in \stateSpace[a] \} \\
            &= \{ \{\propositionSymbol[b] : \propositionSymbol \in \state  \} : \state \in \stateSpace[a] \} \\
            &= \stateSpace[b].
        \end{flalign*}
    
        To prove Property~\ref{itm:property-2}, we go through each proposition $\propositionSymbol$ in the initial state $\initialState[a]$ of the problem PDDL $\planningProblem[a]$, and apply the bijection $f$ to show that we get $\initialState[b]$:
        
        \begin{flalign*}
            &\initialState = \{ \isomorphicMapping(\propositionSymbol[a]) : \propositionSymbol[a] \in \initialState[a] \} \\
            &= \{ \langle \varphi(q^a), \varphi(o_1^a), \ldots, \varphi(o_k^a) \rangle: \propositionSymbol[a] \in \initialState[a] \} \\
            &= \{ \langle q^b, o_1^b, \ldots, o_k^b \rangle: \propositionSymbol[a] \in \initialState[a] \} \\
            &= \{ \propositionSymbol[b]: \propositionSymbol[a] \in \initialState[a] \} \\
            &= \initialState[b].
        \end{flalign*}
    
        To prove Property~\ref{itm:property-3}, we first observe that the \texttt{fullySpecify} function adds all possible trivial propositions to the goal scene graph. By definition, trivial propositions are those that are known not to change the underlying planning problem when added or removed. We can use the same bijection $f$ constructed from $\varphi$. We don't need to add or remove non-trivial predicates, as doing so would change the planning problem by definition. This implies that the augmented goal state represents any other possible goal state in $\goalStates$. With this, we can use the bijection $f$ built from $\varphi$ to apply it to any state in $\goalStates[a]$ and show that we get $\goalStates[b]$:
    
        \begin{flalign*}
            &\goalStates = \{ \{\isomorphicMapping(\propositionSymbol) : \propositionSymbol \in \state  \} : \state \in \goalStates[a] \} \\
            &= \{ \{\langle \varphi(q^a), \varphi(o_1^a), \ldots, \varphi(o_k^a) \rangle \hspace{-.25em}:\hspace{-.25em} \propositionSymbol \in \state  \} \hspace{-.25em}:\hspace{-.25em} \state \in \goalStates[a] \} \\
            &= \{ \langle q^b, o_1^b, \ldots o_k^b \rangle : \propositionSymbol \in \state  \} : \state \in \goalStates[a] \} \\
            &= \{ \{ \propositionSymbol[b] : \propositionSymbol \in \state  \} : \state \in \goalStates[a] \} \\
            &= \goalStates[b].
        \end{flalign*}

        Finally, to prove Property~\ref{itm:property-4}, we observe that \texttt{Equivalent}($P_a$, $P_b$, \textbf{False}) assumes that the domain files are shared between $P_a$ and $P_b$. As a result, any bijection $\varphi$ found by the algorithm must preserve the equality of actions across both problems, since the shared domain enforces consistency in their action sets.

        \noindent \textbf{Case 2:}
        We can use the same arguments from Case 1 with one slight modification. 
        Here, \texttt{fastEquivalent} returns \texttt{True} only if \texttt{isIsomorphic}($\texttt{join(}s_{i,a}, g_a\texttt{)}$, $\texttt{join(}s_{i,b}, g_{b}\texttt{)}$) returns \texttt{True}. 
        The difference with Case 1 is that we join the initial states with the goal propositions scene graph $g$ instead of the fully specified scene graph $g^\star$.
         
        In Case 1 we have proven that when \texttt{isIsomorphic}($p_a$, $p_b$) returns \texttt{True}, the planning problems $\planningProblem[a]$ and $\planningProblem[b]$ are equivalent under the bijection $f$. 
        Since $p_a = \texttt{join(}s_{i,a}, g_{a}^\star\texttt{)}$ and $p_b = \texttt{join(}s_{i,b}, g_{b}^\star\texttt{)}$, we can let $g_{a}^\star = g_{a}$ and $g_{b}^\star = g_{b}$. 
        This is because, by definition, trivial propositions do not change the underlying planning problem when added or removed, and in this case, we are only removing trivial propositions. Hence, by the proof of Case 1 we conclude that this case has also to be true.

    Now, we will prove the backward implication. In this case, we will assume that there exists a bijection $f$ that follows properties~\ref{itm:property-1},~\ref{itm:property-2}, and~\ref{itm:property-3} of Definition~\ref{def:equivalence}.
    By Lemma~\ref{lemm:bijection-transformation}, we know that $f$ comes from $\varphi$, and since both are bijections, we can recover $\varphi$ by doing the inverse process.
    We need to prove that these conditions lead to \texttt{isIsomorphic} (line 16) or \texttt{fastEquivalent} (line 5) returning \texttt{True}.

    Since there exists a bijection $\varphi$ that comes from $f$, the number of objects must be the same for planning problems $\planningProblem[a]$ and $\planningProblem[b]$, as stated in Property~\ref{itm:property-1}.
    With Property~\ref{itm:property-2}, we know that this bijection $\varphi$ must also make the initial scenes isomorphic.
    Assume that the sets of goal propositions have the same size. 
    By using Property~\ref{itm:property-3}, we know that the bijection $\varphi$ would make \texttt{fastEquivalent} (line 5) return \texttt{True}.

    Conversely, assume that the sets of goal propositions have different sizes. 
    In this case, \texttt{canDoFast} would not allow \texttt{fastEquivalent} to execute.
    In the case of \texttt{isIsomorphic} (line 16), we know that $p_a = \texttt{join(} s_{i,a} g_{a}^\star\texttt{)}$ and $p_b = \texttt{join(} s_{i,b}, g_{b}^\star\texttt{)}$.
    Since we know that we can construct the bijection $\varphi$ that maps vertices from the bijection $f$, we also know that since $\varphi$ comes from $f$, this bijection respects initial states (Property~\ref{itm:property-2}) and goal states (Property~\ref{itm:property-3}). 
    Furthermore, the algorithm does not modify any actions in the domain file, and since a single domain is shared across all problems, the actions remain consistent between them (Property~\ref{itm:property-4}). 
    Therefore, \texttt{isIsomorphic}($p_a$, $p_b$) must return \texttt{True}.

    With this, we have proven that if two problems are equivalent under Definition 2, then the equivalence algorithm must return \texttt{True}.

    Since we have proven both directions of the biconditional, Proposition~\ref{prop:is-equivalent-no-placeholder} is true.
\end{proof}

\begin{proposition}
\label{prop:is-equivalent-placeholder}
    \texttt{Equivalent}($P_a$, $P_b$, \textbf{True}) returns \textbf{True} if and only if $P_a$ represents a planning problem that is equivalent (under Definition~\ref{def:equivalence}) to some planning problem represented by $P_b$ after a permutation of the objects in its goal state.
\end{proposition}

\begin{proof}
    To prove this proposition, we need to prove each direction of the biconditional. First, we will prove the forward implication. There are two possible ways the algorithm can return \texttt{True}: Either in \texttt{fastEquivalent} (line 5) or in \texttt{isIsomorphic}($s_{i,a}$, $s_{i,b}$) $\wedge$ \texttt{isIsomorphic}($g^\star_a$, $g^\star_b$) (line 12).

    To prove the case where \texttt{fastEquivalent} (line 5) returns \texttt{True}, we can use the same argument as in Case 1 of the forward implication in the proof of Proposition~\ref{prop:is-equivalent-no-placeholder}. 
    This is because the placeholder being \texttt{True} does not affect the \texttt{fastEquivalent} function.

    Now, assume that \texttt{isIsomorphic}($s_{i,a}$, $s_{i,b}$) $\wedge$ \texttt{isIsomorphic}($g^\star_a$, $g^\star_b$) is \texttt{True}. 
    This implies that there exist two bijections: one $\varphi$ on the vertices $V_{s_{i,a}}$ to $V_{s_{i,b}}$, and another $\psi$ between vertices $V_{g^\star_a}$ to $V_{g^\star_b}$.
    Now, by using Lemma 1, we can build two new bijections, $f$ and $h$, that map between propositions in $L$ instead of nodes in the scene graph, as $\varphi$ and $\psi$ do.

    Since the underlying mappings of $f$ and $h$ can be different, we need a way to combine them into a single bijection. One way to construct this bijection is to permute the mappings of $f$ until they fit the mappings of $h$. This implies that there are many pairs of $f$ and $h$, and each pair can be combined into a global bijection using function composition $f \circ h$.  For each of these permutations, we can then use the arguments in Case 2 of the proof of Proposition~\ref{prop:is-equivalent-no-placeholder} to show that these pairs hold all three properties of Definition~\ref{def:equivalence}.

    Now, we will prove the backward implication. First, since \texttt{fastEquivalent} remains the same, we do not need to prove anything new about it; we can simply reuse the previous results.

    Assume $P_a$ represents a planning problem that is equivalent to some planning problem represented by $P_b$ after a permutation of the objects in its goal state. 
    This implies that there are several bijections $f$ from the predicates in $P_a$ to the predicates in $P_b$. For each of these bijections $f$, we can use Lemma~\ref{lemm:bijection-transformation} to transform it into $\varphi$. 
    Since we know by Property\ref{itm:property-2} that $\varphi$ maps initial states, there is an isomorphism between the initial states. 
    By Property~\ref{itm:property-3}, we know that $\varphi$ maps any goal state of $P_a$ to any other possible state in $P_b$. 
    Hence, $s^\star_a$ has to be isomorphic with $s^\star_b$ since $\varphi$ is a bijection between their nodes.

    Since we have proven both directions of Proposition~\ref{prop:is-equivalent-placeholder}, we know it is true.
\end{proof}

\begin{proof}[Proof of Theorem \ref{thm:main}]
    Having the proof of Proposition~\ref{prop:is-equivalent-no-placeholder} and the proof of Proposition~\ref{prop:is-equivalent-placeholder}, we know that the theorem, which is just the conjunction of these propositions, must be true.
\end{proof}

\begin{lemma}
\label{lemm:bijection-transformation}
    Let $\varphi$ be an isomorphic bijection between nodes in two scene graphs (propositions or objects). 
    There exists a bijection $f$ over the propositions in $L$ such that for any proposition $\propositionSymbol \in \propositionSymbolSpace$, $f(\propositionSymbol) = \langle \varphi(q), \varphi(o_1), \varphi(o_2), \ldots, \varphi(o_k) \rangle$, where $q$ is the predicate and $O = \{ o_1, o_2, \ldots, o_k \}$ is a set of objects.
\end{lemma}

\begin{proof}
    Take any proposition $\propositionSymbol \in \propositionSymbolSpace$ and identify its predicate $q$ and the ordered set of objects it acts on, $O = \langle o_1, o_2, \ldots, o_k \rangle$.
    Then, we can transform the proposition $\propositionSymbol$ into the ordered tuple $\propositionSymbol = \langle q, o_1, o_2, \ldots, o_k \rangle$.
    Finally, we can build $f(\propositionSymbol) = \langle \varphi(q), \varphi(o_1), \varphi(o_2), \ldots, \varphi(o_k) \rangle$, which is a bijection over the propositions in $L$.
\end{proof}

\section{Examples of Algorithm 1}
\label{app:examples}

\begin{figure*}[h!] 
    \centering
    \includegraphics[width=\linewidth]{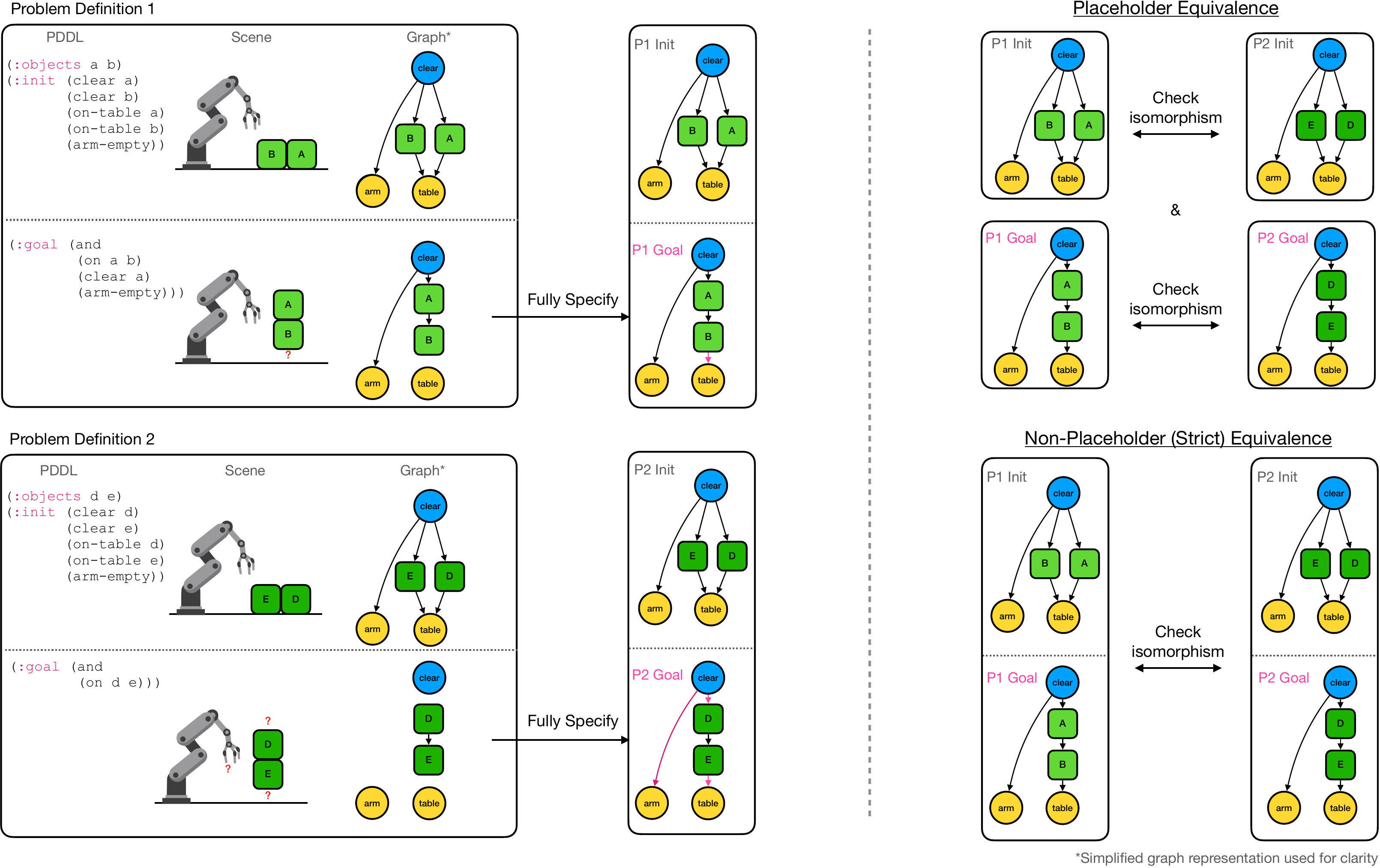}
    \caption{An illustration of the algorithm to check if two PDDL problems are equivalent. It shows each of the stages of the algorithm: transforming to scene graphs, fully specifying the goal propositions, and checking for graph isomorphism.}
    \label{fig:algorithm-example}
\end{figure*}

Figure~\ref{fig:algorithm-example} shows examples of how Algorithm~\ref{alg:planning-problem-equivalency} checks whether two PDDL problems are equivalent.

\section{Implementation of \texttt{fullySpecify}}
\label{app:implementation}

The \texttt{fullySpecify} function adds trivial edges (edges that must exist but are not currently present) to our scene graphs. We build the \texttt{fullySpecify} function using our domain knowledge of the Blocks World and Gripper domains.

\paragraph{Blocks World}
The following helpful facts are true about Blocks World:
\begin{itemize}
    \item If all blocks have its behavior above it defined (they are either \textbf{\texttt{clear}}, have something \textbf{\texttt{on}} top of it, or is being held), then any block that does not have its bottom behavior specified, must be on the table (\textbf{\texttt{on-table}}).
    \item If all blocks have its behavior beneath it defined (they are either \textbf{\texttt{on}} top of something else, are on the table (\textbf{\texttt{on-table}}), or are being held), then any block that does not have its behavior above it defined must be \textbf{\texttt{clear}}.
    \item Following the last two facts: if there is a block \textbf{A} which has its behavior above it unspecified and block \textbf{B} which has its behavior below it unspecified, then the only way that (\textbf{\texttt{on B A}}) can exist is if there isn't already a chain of predicates that leads \textbf{A} to \textbf{B} or vice versa.
    If there already exists a chain connecting these two blocks, that means \textbf{A} and \textbf{B} are already connected by being in the same tower, and that they are the top and bottom of the tower they are a part of.
    \item If there are no ``floating'' blocks (blocks that have both its top and bottom behavior defined), and the arm's behavior is undefined (neither \textbf{\texttt{arm-empty}} nor \textbf{\texttt{holding}}) is present, then the arm must be empty.
    This is because there is no block that the arm could possibly hold here.
\end{itemize}

Using these rules, we can add all possible trivial edges, until no more trivial edges can be found.
Further, these facts will discover all possible trivial edges in Blocks World, meaning the fully specified scene graph will be in its canonical form after \texttt{fullySpecify}.

\paragraph{Gripper}
Similarly, we can build a set of facts that operate on the Gripper domain:

\begin{itemize}
    \item If all balls are assigned a room (\textbf{\texttt{at ball room}}), then all unassigned grippers (no (\textbf{\texttt{free gripper}}) or (\textbf{\texttt{carry ball gripper}})) must be \textbf{\texttt{free}}.
    \item If there is only one room, and if all grippers are already specified (either (\textbf{\texttt{free}}) or (\textbf{\texttt{carry ball gripper}})), then any unspecified balls must be in the only room that exists.
    \item If there is only one room, the robby must be in it (\textbf{\texttt{at-robby onlyRoom}}).
\end{itemize}

These are the only rules we can find.
This is because, if a ball's position is unspecified, it can always be assigned to any arbitrary room.
If the robby is unspecified (no \textbf{\texttt{at-robby}}), then it too, can always be specified to any arbitrary room.

\paragraph{Floor Tile} In the standard Floor Tile, we know that:
\begin{itemize}
    \item A tile cannot be unpainted.
    \item All available colors (\textbf{\texttt{available-color color}}) are unchangeable and cannnot be added.
    \item Relative positional propositions (\textbf{\texttt{right} tile1 tile2}) and (\textbf{\texttt{up} tile1 tile2}) are unchangeable and cannot be added.
\end{itemize}

Further, given our simplifications, we have the following characteristics in our domain:
\begin{itemize}
    \item Traversal is bidirectional.
    \item Painting a tile and other robots do not block path-finding.
    \item From the last two facts, robots cannot block themselves off through actions, so the only way the position of a robot is fixed is if:
    \begin{enumerate}
        \item It is explicitly defined in the goal description.
        \item The robot is on an isolated tile to begin with.
    \end{enumerate}
\end{itemize}

Applying these rules allows us to fully specify a Floor Tile goal description.

\section{Dataset Details}
\label{app:dataset-details}
To obtain interesting problems for each domain, we crafted a set of tasks instead of randomly generating instances. 
Below are descriptions of each domain and their corresponding task configurations.

\begin{figure}
    \centering
    \includegraphics[width=\linewidth]{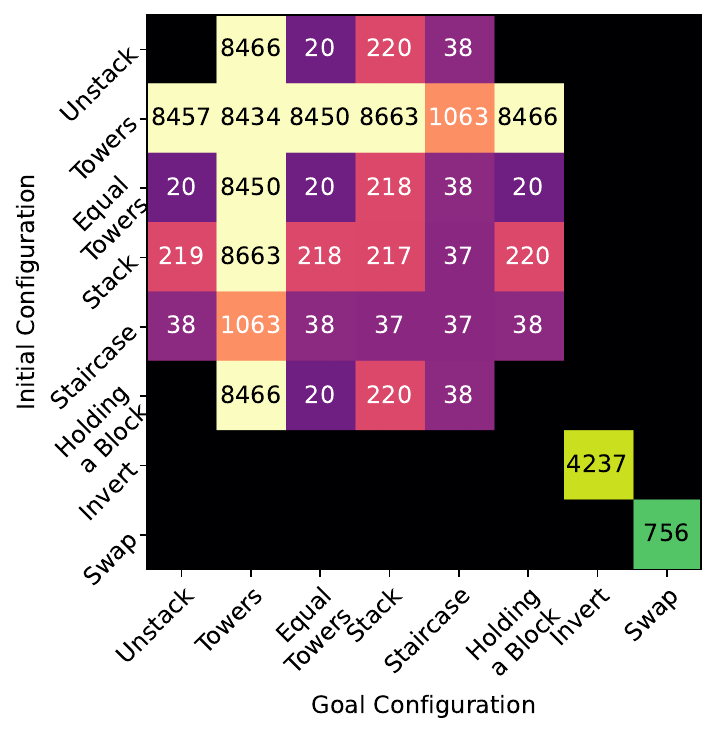}
    \caption{A visualization of the different tasks in Blocks World. Training set displayed in ``magma'' color map (purple to yellow). Test set displayed in ``viridis'' color map (blue to yellow).}
    \label{fig:blocksworld-tasks}
\end{figure}

\noindent \textbf{Blocks World.}
This domain involves using a robotic hand to manipulate a set of blocks arranged in various configurations on a table.  Blocks World configurations:
\begin{itemize}
    \item \textbf{Stacked}: All the blocks are in a single tower.
    \item \textbf{Unstacked}: All blocks are separately placed on the table.
    \item \textbf{Holding a Block}: All blocks except one are separately placed on the table, the robot holds the last block.
    \item \textbf{Staircase}: Towers are in incrementing heights (1, 2, 3, $\ldots$).
    \item \textbf{Equal Towers}: Blocks are in towers of equal height.
    \item \textbf{Swap} (tied): Goal and initial configurations are tied: given two towers, swap the base blocks, and leave the rest of the structure unchanged.
    \item \textbf{Invert} (tied): Goal and initial configurations are tied: Given a set of stacks, flip each stack.
    \item \textbf{Towers}: Build towers to match a specified list of heights.
\end{itemize}

\begin{figure}
    \centering
    \includegraphics[width=\linewidth]{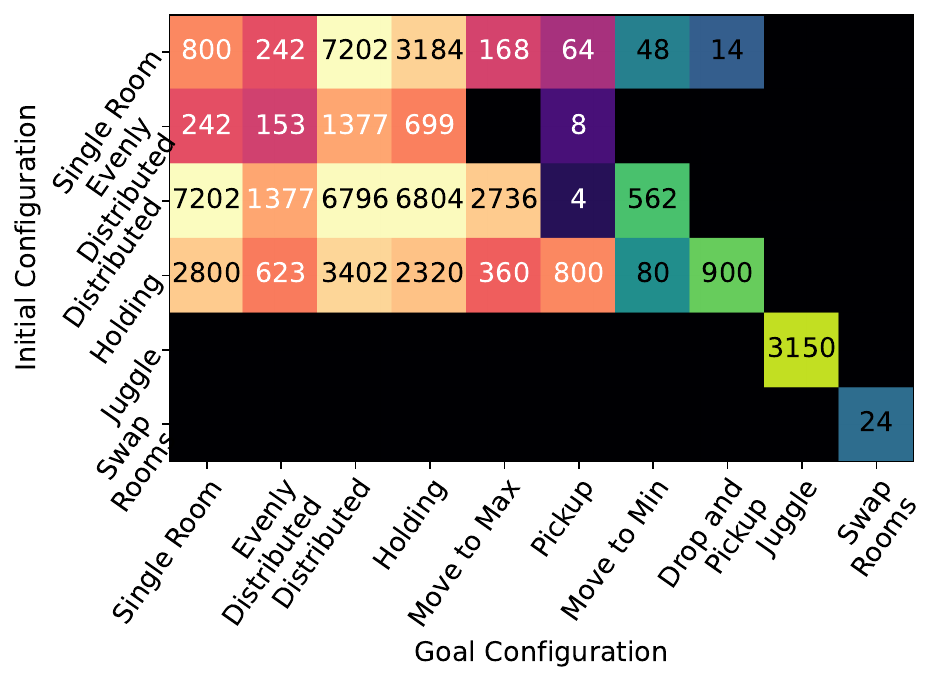}
    \caption{A visualization of the different tasks in Gripper. Training set displayed in ``magma'' color map (purple to yellow). Test set displayed in ``viridis'' color map (blue to yellow).}
    \label{fig:gripper-tasks}
\end{figure}

\noindent \textbf{Gripper.} 
This domain features a robot equipped with grippers, tasked with transporting balls from one room to another. Gripper Configurations:
\begin{itemize}
    \item \textbf{Single Room}: All balls are in one room.
    \item \textbf{Evenly Distributed}: Each room has the same number of balls.
    \item \textbf{Distribute}: Each room has a predefined number of balls.
    \item \textbf{Swap Rooms} (tied): Goal and initial configurations are tied: swap ball locations between two rooms.
    \item \textbf{Move to Max} (goal only): Move all balls to the room that started with the most number of balls.
    \item \textbf{Move to Min} (goal only): Move all balls to the room that started with the fewest balls.
    \item \textbf{Pickup} (goal only): Pick up all the balls
    \item \textbf{Drop and Pickup} (goal only): Drop all the balls in the first room, and pick up all the balls that started in rooms.
    \item \textbf{Juggle} (tied): Goal and initial configurations are tied: given that the robot is holding some balls, shuffle which arm holds which ball.
\end{itemize}

\begin{figure}
    \centering
    \includegraphics[width=\linewidth]{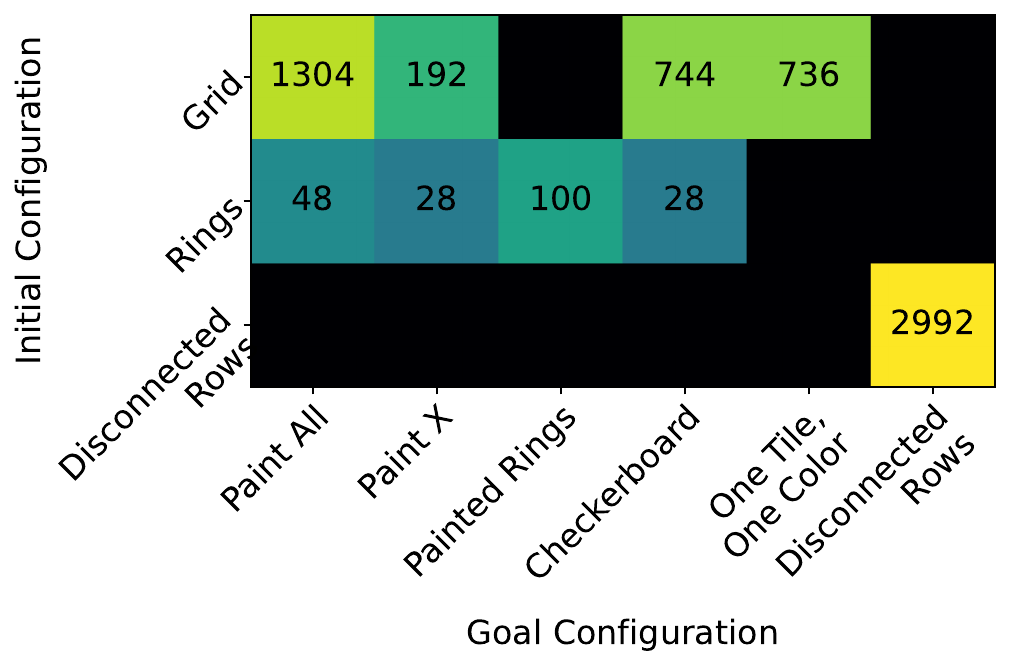}
    \caption{A visualization of the different tasks in Floor Tile. Training set displayed in ``magma'' color map (purple to yellow). Test set displayed in ``viridis'' color map (blue to yellow).}
    \label{fig:floor-tile-tasks}
\end{figure}

\noindent \textbf{Floor Tile.} 
This domain features a group of robots tasked with painting some tiles. Floor Tile configurations:
\begin{itemize}
    \item \textbf{Grid} (initial only): Tiles are arranged in a grid.
    \item \textbf{Rings} (initial only): Tiles are arranged in a grid in concentric rings.
    \item \textbf{Disconnected Rows} (tied): Goal and initial configurations are tied: tiles are in rows, but the rows are disconnected from one another.
    The goal is to paint both ends of each row. 
    \item \textbf{Painted Rings} (goal only): Tiles are arranged in a grid in concentric rings, and each ring is painted.
    \item \textbf{Paint All} (goal only): All tiles are painted with one color.
    \item \textbf{Painted X} (goal only): The robot must paint an X across all the tiles.
    \item \textbf{One Tile, One Color} (goal only): Each tile is painted a different color.
    \item \textbf{Checkerboard} (goal only): The tiles are painted in a checkerboard pattern with two colors.
\end{itemize}

\section{License Information}
\label{app:licence}
The data for \benchmarkName\ is available on HuggingFace\footnote{\href{https://huggingface.co/datasets/BatsResearch/planetarium}{\texttt{hf.co/datasets/BatsResearch/planetarium}}}
and released under a Creative Commons CC-BY-4.0 license.
In addition, all code to create the dataset and evaluate models on the benchmark is available on GitHub\footnote{\href{https://github.com/BatsResearch/planetarium}{\texttt{github.com/BatsResearch/planetarium}}}
and released under a BSD 3 license.
The authors are responsible for the content of the dataset.

In our study, we use the Blocksworld, Gripper, and Floor Tile domains, which are sourced from the IPC~\cite{MCDERMOTT-2000-IPC}.
The specific domain files used in our research are obtained from \texttt{pddl-generators} ~\cite{SEIPP-2022-PDDL-GEN}, which is distributed under the GNU General Public License.
The domain file for Floor Tile is a modified version of the one found from \texttt{pddl-generators} ~\cite{SEIPP-2022-PDDL-GEN}

GitHub Copilot was used to assist in writing the code for this paper.

\section{Fine-tuning}
\begin{table}[H]
    \captionsetup{skip=10pt}
    \centering
    \begin{tabular}{l||l}
    \toprule
    Hyperparameters & Value\\
    \midrule
        Optimizer &adamw\_torch\\
        Learning rate & 2e-5\\
        Batch Size & 1\\
        Betas & (0.9, 0.999)\\
        Epsilon & 1e-8\\
        Weight Decay & 0.01\\
    \midrule
        Max Sequence Length & 1500\\
    \midrule
        LoRA rank  &16\\
        LoRA alpha &32\\
        LoRA Dropout &0.05\\
    \bottomrule
    \end{tabular}
    \caption{Training hyperparameters when fine-tuning all models included in the paper. The hyperparameters are separated into three parts: supervised fine-tuning, generation-related, and LoRA-related. }
    \label{tab:hyperparams}
\end{table}

We fine-tuned the open-weight models using QLoRA with a rank of $16$, adhering to the hyperparameter recommendations for small models provided by the original authors \cite{dettmers2023qlora}. Models were loaded with 4-bit precision and trained over the training set for a single epoch. Fine-tuning for each model used either two NVIDIA GeForce RTX 3090 GPUs or two NVIDIA A6000 GPUs, operating with data parallelization for approximately 15 hours. We truncate the longest $5\%$ of our training dataset due to GPU memory constraints.

\label{app:experiments}
Table \ref{tab:hyperparams} displays the hyperparameters used for fine-tuning across all models. 
In addition, models are loaded using 4-bit NF4 quantization with double quantization to reduce the average memory footprint. We also use the bfloat16 compute data type for faster training.

\section{Maintenance Plan}
\label{app:maintenance-plan}

The research group that created \benchmarkName\ will maintain the \benchmarkName\ dataset and benchmark with an open GitHub repository and issue submission system, and the dataset will be hosted on HuggingFace. 
The maintenance plan includes regular issue tracking, with reviews and categorization of issues, aiming to resolve high-priority issues within a week and scheduling minor updates quarterly.  
Comprehensive documentation and automated testing will ensure quality and compatibility.

\onecolumn
\section{Text-to-PDDL Pair Example}
\label{app:text-to-PDDL-pair}
\begin{tcolorbox}[collower=white,
    valign lower=top,
    toptitle=1.5mm,
    bottomtitle=1.5mm,
    title={BLOCKS WORLD}]\textbf{Ground Truth PDDL}
    \begin{lstlisting}
    (define (problem equal_towers_to_equal_towers_5)
        (:domain blocksworld)
        (:requirements :strips)
        (:objects b1 b2 b3 b4 b5)
        (:init (arm-empty) 
               (clear b5) 
               (on b2 b1) 
               (on b3 b2) 
               (on b4 b3) 
               (on b5 b4) 
               (on-table b1))
        (:goal (and (arm-empty) 
                    (on-table b1) 
                    (on b2 b1) 
                    (on b3 b2) 
                    (on b4 b3) 
                    (on b5 b4) 
                    (clear b5)))
    )
    \end{lstlisting}
    \textbf{Natural Language Description} \\
    {\color{red} (objects:} You have 5 blocks, b1 through b5, stacked into 1 towers of equal heights, and your arm is empty.
    {\color{red} )}
    {\color{olive} (abstract init:} stacked into 1 towers of equal heights, and your arm is empty.
    {\color{olive} | explicit init:} Your arm is empty. b1 is on the table. b2 is on b1. b3 is on b2. b4 is on b3. b5 is on b4. b5 is clear.
    {\color{olive} )}
    Your goal is to have the following: 
    {\color{blue} (abstract goal:} stack the blocks into 1 towers of equal heights.
    {\color{blue} | explicit goal:} Your arm should be empty. b1 should be on the table. b2 should be on b1. b3 should be on b2. b4 should be on b3. b5 should be on b4. b5 should be clear.
    {\color{blue} ) }
\end{tcolorbox}
\twocolumn

\end{document}